\documentclass[conf]{new-aiaa}

\usepackage[utf8]{inputenc}
\usepackage{algorithm}
\usepackage{graphicx}
\usepackage{amsmath}
\usepackage[version=4]{mhchem}
\usepackage{siunitx}
\usepackage{algpseudocode}
 \usepackage{amsthm}
\usepackage{subfigure}
 \theoremstyle{definition}

\DeclareUnicodeCharacter{221E}{\ensuremath{\infty}}

\newtheorem{proposition}{Proposition}
\usepackage{longtable,tabularx}
\usepackage{mathtools}
\newcommand{\floor}[1]{\left \lfloor #1 \right \rfloor} 

\title{Fast and Safe Aerial Payload Transport in Urban Areas}

\author{Aeris El Asslouj \footnote{Student, Electrical and Computer Engineering Department, University of Arizona, Email: aymaneelasslouj@email.arizona.edu.}, Harshvardhan Uppaluru \footnote{PhD Student, Aerospace and Mechanical Engineering Department, University of Arizona, Email: huppaluru@email.arizona.edu.} and Hossein Rastgoftar \footnote{Assistant Professor, Aerospace and Mechanical Engineering Department, University of Arizona, Email: hrastgoftar@arizona.edu.}}

\begin{document}

\maketitle

\begin{abstract}
This paper studies the problem of fast and safe  aerial payload transport by a single quadcopter in urban areas. The quadcopter payload system (QPS) is considered as a rigid body and modeled with a nonlinear dynamics. The urban area is modeled as an obstacle-laden environment with obstacle geometries obtained by incorporating realistic LIDAR data. Our approach for payload transport is decomposed into high-level motion planning and low-level trajectory control. For the low-level trajectory tracking, a feedback linearization control is applied to stably track the desired trajectory of the quadcopter. For high-level motion planning, we integrate A* search and polynomial planning to define a safe trajectory for the quadcopter assuring collision avoidance, boundedness of the quadcopter rotor speeds and tracking error, and fast arrival to a target destination from an arbitrary initial location. 
\end{abstract}

\section{Nomenclature}

{\renewcommand\arraystretch{1.0}
\noindent\begin{longtable*}{@{}l @{\quad=\quad} l@{}}
$s_i$  & Angular speed of rotor $i$ ($i=1,2,3,4$) \\
$s_{max}$  &Maximum rotor speed\\
$\mathbf{r}_i$ &    Initial position \\
$\mathbf{r}_f$& Target position \\
$\mathbf{r}(t)$& Actual trajectory \\
$\mathbf{p}(t)$& Desired trajectory \\
$\delta$& Tracking error upper bound \\
$\phi$, $\theta$, $\psi$& Roll, pitch, and yaw angles of the quadcopter \\
$m$& Quadcopter mass \\
$\mathbf{J}$& Mass moment of inertia of the quadcopter \\
$\mathbf{S}$& Rotation matrix \\
$\hat{\mathbf{e}}_1$, $\hat{\mathbf{e}}_2$, $\hat{\mathbf{e}}_2$& Bases of the inertial coordinate system \\
$\hat{\mathbf{b}}_1$, $\hat{\mathbf{b}}_2$, $\hat{\mathbf{b}}_2$& Bases of the quadcopter body coordinate system \\
$\boldsymbol{\omega}$& Quadcopter angular velocity\\
\end{longtable*}}

\section{Introduction}

Over the past two decades, quadcopters has become increasingly affordable and widely used for military and non-military applications due to their high performance, maneuverability and dynamic simplicity. Applications of quadcopters include remote sensing, firefighting, traffic surveillance tasks \cite{puri2005survey, kanistras2013survey}, search and rescue operations \cite{surmann2019integration, polka2017use, tomic2012toward}, wildlife monitoring and exploration \cite{witczuk2018exploring} and educational purposes. Aerial payload transportation is one such interesting application commonly used in construction, military response, emergency response, and delivery tasks \cite{mathew2015planning, arbanas2016aerial}. Typically, for cases when the payload is not heavy, a single quadcopter can carry the payload via a single cable attached between quadcopter and payload. This is particularly useful in remote areas with uneven terrains where it is difficult to secure a safe landing place.

\subsection{Related Work}

Previously, the area of payload transport has been extensively studied for helicopters \cite{cicolani1995simulation, bernard2009generic}. The dynamics, stabilization, and control of a payload carrying helicopter were modeled in \cite{pounds2012stability, bernard2009generic, oh2006dynamics}. So far, single quadcopter or multiple quadcopters have been considered for payload transport and deployment \cite{palunko2012agile, michael2011cooperative, maza2009multi, mellinger2013cooperative, rastgoftar2018cooperative, acosta2020accurate} due to their high thrust generation capabilities. Generally, there are two means of payload transportation carried out by quadcopters, i.e., active and passive attachments, each having their own advantages and disadvantages. The passive approach uses a suspended cable \cite{tang2015mixed, yang2019energy, sreenath2013trajectory} with one end attached to the quadcopter and the load is attached to the other end. However, this approach is not feasible in outdoor environments and the controller design becomes more complex due to additional degree of under-actuation. Since most passive approaches are based on the common assumption that the suspended cable is always taut, their applications are restricted. Active approach adds an additional degree of freedom that requires a gripper to grasp at the payload and provides a better solution especially in constrained altitudes.

Aerial payload transportation and manipulation using a single quadcopter was studied \cite{kim2013aerial, guerrero2015passivity, guo2017mixed, goodarzi2016autonomous, yang2019energy}. 
Stabilization of a single quadcopter carrying a single payload was also analyzed \cite{ sreenath2013geometric, goodarzi2014geometric},  Researcher have proposed H$^\infty$ control \cite{guo2017mixed}, and PID control \cite{barawkar2017admittance} for a quadcopter carrying a suspended payload. A quadcopter carrying payload with varying length cable was studied \cite{ goodarzi2016autonomous}. For emergency response, commercial and military applications, cooperative aerial payload transportation and manipulation has been considered \cite{mellinger2011design, kim2013aerial, michael2009kinematics, sreenath2013dynamics}


\subsection{Contributions and Outline}

This paper proposes a multi-layer approach for safe and fast transportation of an aerial payload carried by a single quadcopter in an urban area. We assume that the quadcopter and payload together act as a rigid body and is known as quadcopter payload system (QPS). QPS is modeled by $14$-th order nonlinear dynamics whereas the environment is modeled using geometry from LIDAR data. 
We tested our algorithm and controller in a simulation with a sample payload transport mission through the University of Arizona and presented the results in \ref{sec:simulation}. We used The United States Geological Survey (USGS)'s Lidar data to create the simulation environment in a manner that is widely applicable in the United States as the USGS's Lidar data covers most of the country. Compared with the existing literature, our proposed payload transportation solution offers the following contributions:
\begin{enumerate}
    \item high-level motion planning that integrates A* search and polynomial planning to obtain safe trajectory of the QPS minimizing travel distance from an initial position to a target destination.
    \item low-level trajectory tracking control ensuring stability and boundedness of rotor angular speeds and tracking error in a general payload transportation mission in an obstacle-laden environment with arbitrary distributions of obstacles.
\end{enumerate}

\bigskip

This paper is organized as follows: The problem statement is defined in Section \ref{sec:problemstatement}. Section \ref{sec:modeling} describes the model of the environment and the quadrotor payload system. Section \ref{sec:control} discusses the mathematical modeling of quadcopters and trajectory tracking control. Section \ref{sec:planning} presents the high-level motion planning approach used for simulations. We finally present our simulation results using the described model and control in Section \ref{sec:simulation} before putting forward our concluding remarks in Section \ref{sec:conclusion}.

\section{Problem Statement}
\label{sec:problemstatement}
We consider a quadcopter carrying a payload in an urban area with given initial position $\mathbf{r}_i$ and target position $\mathbf{r}_f$. The QPS is enclosed by a sphere of radius $\epsilon$ and follows the nonlinear dynamics presented in Section \ref{sec:Quadcopter-payload system model}. The environment is made up of free space and obstacle space (i.e terrain and structures) as described in Section \ref{sec:Environment model}.

This paper develops a pair of algorithms to choose a valid desired trajectory $\mathbf{p}(t)$ in Section \ref{sec:planning}. They include A*-based algorithm for spatial planning to keep desired trajectory $\mathbf{p}(t)$ collision-free, and a bi-section-based algorithm for temporal planning of $\mathbf{p}(t)$ to minimize mission time. Then, a feedback linearization controller will be applied in Section \ref{sec:control} to compute rotor speeds that allow the system to track the desired trajectory while respecting the following safety conditions:

\textbf{Bounded Rotor Speed:} The rotor angular speeds of the quadcopter, denoted by $s_1$ through $s_4$ need to satisfy the safety requirement
\begin{equation}\label{boundedrotor}
    \bigwedge_{j=1}^4\left(0\leq s_j\left(t\right)\leq s_{max}\right),\qquad \forall t,
\end{equation}
where $s_{max}$ is the maximum angular speed for all quadcopter rotors, and ``$\bigwedge$''
 means ``include all''.

\bigskip
 
\textbf{Bounded Trajectory Tracking:} It is required to assure that the tracking error remains bounded at any time $t$. This condition can be formally specified by
\begin{equation}\label{boundedcontrol}
   \|\mathbf{r}(t)-\mathbf{p}(t)\|\leq \delta,\qquad \forall t,
\end{equation}
where $\delta$ is the tracking bound and $\mathbf{r}(t)$ is the actual trajectory of the QPS. 

\bigskip

\textbf{Obstacle Collision Avoidance Guarantee:} It is required to guarantee that the QPS does not hit any structure or terrain in the urban area. This is  formally specified as
\begin{equation}\label{obstacleavoidance}
   \|\mathbf{r}(t)-\mathbf{o}\| > \epsilon,\qquad \forall \mathbf{o} \in \mathbf{Obstacle\ space},\qquad \forall t,
\end{equation}
where $\mathbf{Obstacle\ space}$ is the space occupied by either the terrain or structures.

\bigskip

\section{Modeling}
\label{sec:modeling}
We first define the model of the environment in Section \ref{sec:Environment model} and the model of the QPS in Section \ref{sec:Quadcopter-payload system model}. In Section \ref{sec:Quadcopter-payload system model}, we also derive a set of formulas to compute desired rotor speeds that guarantee a desired thrust and a desired Euler acceleration.

\subsection{Environment model}
\label{sec:Environment model}

The environment is modeled as a continuous three-dimensional space $\mathbf{Environment}$ with coordinate system $\left(x,y,z\right)$. The environment is split by an elevation map $M$ into $\mathbf{Free\ space}$ above it and $\mathbf{Obstacle\ space}$ under it. This representation is easier to process and only requires altitude data but does not allow for crossing under structures. It can be formalized as:

\begin{subequations}
\begin{equation}
M : (x,y) \rightarrow \text{highest terrain or structure altitude at (x,y)}, \qquad \forall (x,y,z) \in \mathbf{Environment},
\end{equation}
\begin{equation}
\mathbf{Obstacle\ space} = \left\{(x,y,z) \in \mathbf{Environment} \ \vert \ z < M(x,y)\right\},
\end{equation}
\begin{equation}
\mathbf{Free\ space} = \left\{(x,y,z) \in \mathbf{Environment} \ \vert \ z > M(x,y)\right\}.
\end{equation}
\end{subequations}

\bigskip

To ensure that the system remains at a distance of $\epsilon + \delta$ from the obstacle space, we expand the elevation map $M$ into an expanded elevation map $M_E$. This is done by taking each point of the surface formed by the elevation map in three dimensional space and translating it in the local upward normal direction by a distance of $\epsilon + \delta$. If we define $\mathbf{S_M}$ and $\mathbf{Exp(S_M)}$ as respectively the three dimensional surface of the elevation map and its expanded version, we can write

\begin{equation}
    \mathbf{Exp(S_M)} = \left\{(x,y,z) + (\epsilon + \delta)\hat{\mathbf{n}}_{\mathbf{S_M}}(x,y) \ \vert \ (x,y,z) \in \mathbf{S_M} \right\},
\end{equation}
where $\hat{\mathbf{n}}_{\mathbf{S_M}}(x,y)$ is the unit normal of $\mathbf{S_M}$ at $(x,y)$ such that:
\begin{equation}
    \hat{\mathbf{n}}_{\mathbf{S_M}}(x,y) = \frac{1}{\left\lVert \begin{bmatrix}-\frac{\partial M}{\partial x} & -\frac{\partial M}{\partial y} & 1 \end{bmatrix} \right\rVert}\begin{bmatrix}-\frac{\partial M}{\partial x} & -\frac{\partial M}{\partial y} & 1 \end{bmatrix}.
\end{equation}

The expanded elevation map $M_E$ is defined as the highest point of the expanded elevation surface $\mathbf{Exp(S_M)}$:
\begin{equation}
M_E : (x,y) \rightarrow max\left\{z \ \vert \ (x,y,z) \in \mathbf{Exp(S_M)} \right\}.
\end{equation}
Note that because of this, $\mathbf{Exp(S_M)}$ is not necessarily the surface $\mathbf{S_{M_E}}$ formed by $M_E$ in three dimensional space.

\bigskip

In a similar fashion to $M$, $M_E$ splits $\mathbf{Environment}$ into $\mathbf{Restricted\ free\ space}$ above it and $\mathbf{Expanded\ obstacle\ space}$ under it:

\begin{subequations}
\begin{equation}
\mathbf{Expanded\ obstacle\ space} = \left\{(x,y,z) \in \mathbf{Environment} \ \vert \ z < M_E(x,y)\right\},
\end{equation}
\begin{equation}
\mathbf{Restricted\ free\ space} = \left\{(x,y,z) \in \mathbf{Environment} \ \vert \ z > M_E(x,y)\right\}.
\end{equation}
\end{subequations}

\subsection{Quadcopter-payload system (QPS) model}
\label{sec:Quadcopter-payload system model}

\begin{figure}[ht]
\centering
\includegraphics[width=4in]{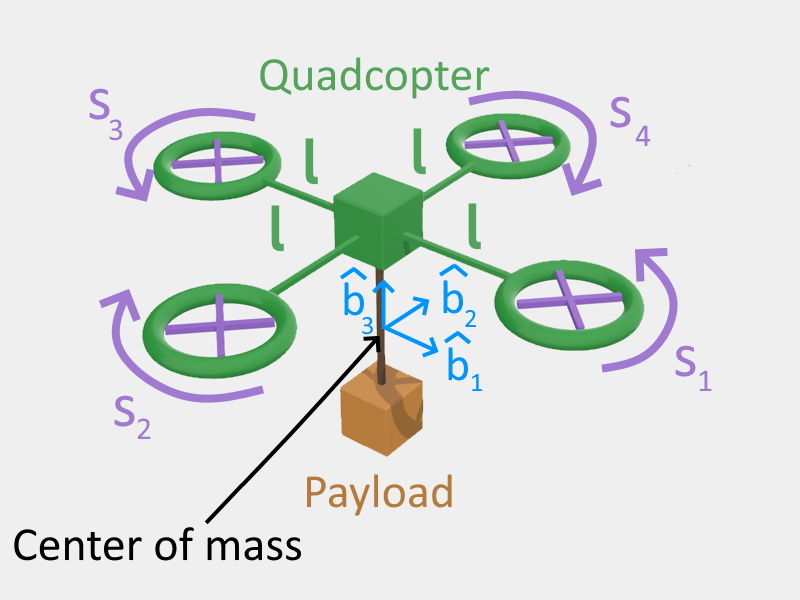}
\caption{QPS model with rotor angular speeds $s_1$ through $s_4$, rotor arm length $l$, and frame $(\hat{\mathbf{b}}_1,\hat{\mathbf{b}}_2,\hat{\mathbf{b}}_3)$}
\label{quad_pic}
\end{figure}

\begin{figure}[ht]
\centering
\includegraphics[width=4   in]{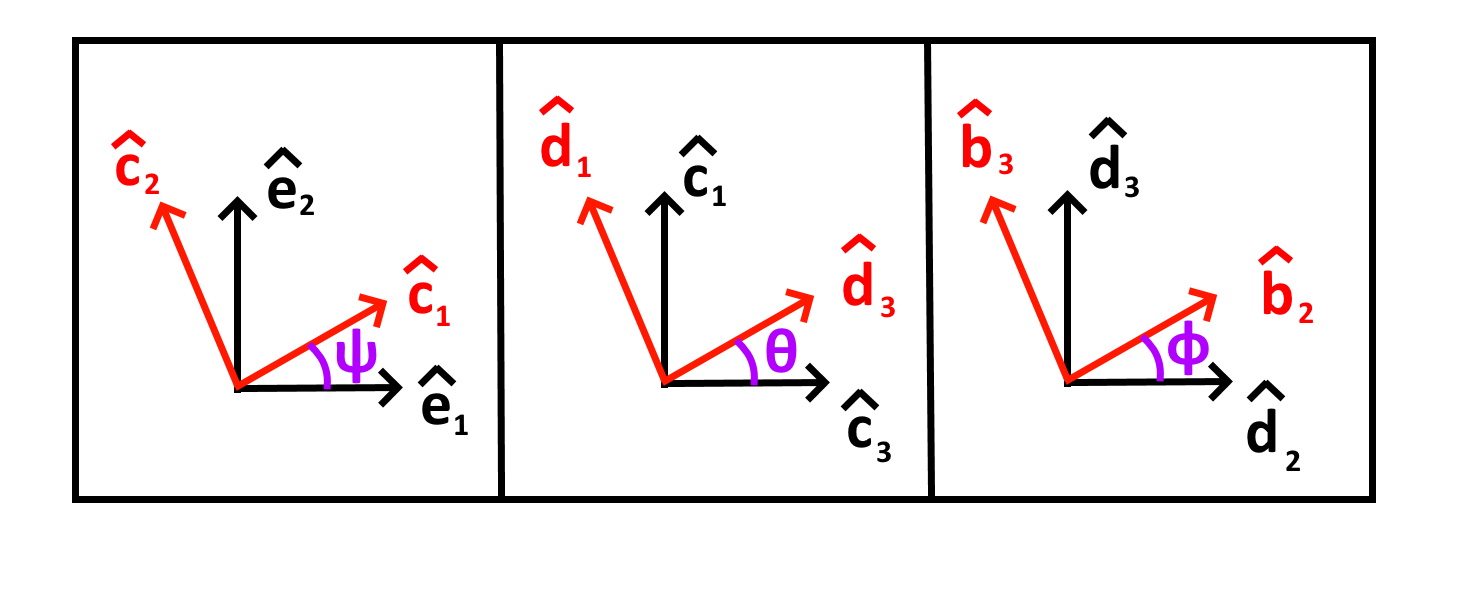}
\caption{$3-2-1$ standard for rotation with $\phi$, $\theta$, and $\psi$ as the roll, pitch, and yaw angles}
\label{rot_pic}
\end{figure}

We use the kinematics and dynamics presented in Sections \ref{sec:kinematics} and \ref{sec:dynamics} to model the motion of the QPS in an urban area.
\subsubsection{Kinematics}
\label{sec:kinematics}
To characterize the rotation of the QPS, we fix a body frame at the center of mass of the QPS with the schematic shown in Fig. \ref{quad_pic} ,  and apply the $3-2-1$ Euler angle standard as shown in Fig. \ref{rot_pic} to characterize the rotation of the QPS with respect to the inertial coordinate system which is specified by matrix
 \[
\mathbf{S}\left(\phi,\theta,\psi\right)= \begin{bmatrix}
    \cos{\theta} \cos{\psi}&\cos{\theta} \sin{\psi} &-\sin{\theta}\\
  \sin{\phi}\sin{\theta} \cos{\psi}-\cos{\phi}\sin{\psi}&\sin{\phi}\sin{\theta} \sin{\psi}+\cos{\phi}\cos{\psi}&\sin{\phi}\cos{\theta}\\
  \cos{\phi}\sin{\theta} \cos{\psi}+\sin{\phi}\sin{\psi} &\cos{\phi}\sin{\theta} \sin{\psi}-\sin{\phi}\cos{\psi}&\cos{\phi}\cos{\theta}
\end{bmatrix}
,
\]
where $\phi$, $\theta$, and $\psi$ are the roll, pitch, and yaw angles. The QPS body axes are denoted by $\hat{\mathbf{b}}_1$, $\hat{\mathbf{b}}_2$, and $\hat{\mathbf{b}}_3$ and related to the base vectors of the inertial coordinate system by
 \begin{equation}
\hat{\mathbf{b}}_{h}(t)=\mathbf{S}^T\left(\phi,\theta,\psi\right)
\hat{\mathbf{e}}_h,\qquad ~h=1,2,3.
\end{equation}
To obtain the angular velocity and angular acceleration of the QPS, we define two intermediate coordinate transformations with bases $\left(\hat{\mathbf{c}}_1,\hat{\mathbf{c}}_2,\hat{\mathbf{c}}_3\right)$ and $\left(\hat{\mathbf{d}}_1,\hat{\mathbf{d}}_2,\hat{\mathbf{d}}_3\right)$ that are defined as follows:
\begin{subequations}
\begin{equation}
    \hat{\mathbf{c}}_{h}(t)=\mathbf{S}^T\left(0,0,\psi\right)
\hat{\mathbf{e}}_h,\qquad ~h=1,2,3.
\end{equation}
\begin{equation}
    \hat{\mathbf{d}}_{h}(t)=\mathbf{S}^T\left(0,\theta,\psi\right)
\hat{\mathbf{e}}_h,\qquad ~h=1,2,3.
\end{equation}
\end{subequations}
The angular velocity and acceleration of the QPS are then obtained as follows \cite{rastgoftar2021safe}:
\begin{subequations}
\begin{equation}
    {\boldsymbol{\omega}}=\dot{\psi}\hat{\mathbf{c}}_3+\dot{\theta}\hat{\mathbf{d}}_{2}+\dot{\phi}_i\hat{\mathbf{b}}_1,
\end{equation}
\begin{equation}
    \dot{{\boldsymbol{\omega}}}= \ddot{\psi}\hat{\mathbf{c}}_3+\ddot{\theta}\hat{\mathbf{d}}_{2}+\ddot{\phi}_i\hat{\mathbf{b}}_1
     + \dot{\theta}\dot{\psi}\hat{\mathbf{c}}_3\times \hat{\mathbf{d}}_2+\dot{\phi}\left(\dot{\psi}\hat{\mathbf{c}}_3+\dot{\theta}\hat{\mathbf{d}}_2 \right)\times \hat{\mathbf{b}}_1
\end{equation}
\end{subequations}

\subsubsection{Dynamics} 
\label{sec:dynamics}
 By applying the Newton's second law, the translational dynamics of the QPS is obtained by
 \begin{equation}
     m\ddot{\mathbf{r}}=p\hat{\mathbf{b}}_3-mg\hat{\mathbf{e}}_3
 \end{equation}
 where $m$ is the total mass of the QPS, $p$ is the magnitude of the thrust force generated by the rotors, and $g=9.81m/s^2$ is the gravitational acceleration.
To obtain the rotational dynamics, we notice that the QPS is symmetrically distributed around the $\hat{\mathbf{b}}_3$ axis and is symmetric with respect to the $\hat{\mathbf{b}}_3-\hat{\mathbf{b}}_1$ and $\hat{\mathbf{b}}_2-\hat{\mathbf{b}}_3$ planes. 
Therefore, the mass moment of inertia of the QPS is diagonal and positive definite and denoted by $\mathbf{J}$ when it is realized with respect to the QPS body frame. The rotational dynamics of the QPS is then obtained by
\begin{equation}
    \mathbf{J}\dot{{\boldsymbol{\omega}}}=-{\boldsymbol{\omega}}\times \mathbf{J}{\boldsymbol{\omega}}+\boldsymbol{\tau},
\end{equation}
where $\boldsymbol{\tau}$ is the control torque expressed with respect to the inertial coordinate system.
\paragraph{Rotor Angular Speeds:}
To obtain the rotor angular speeds, we first express control torque vector $\boldsymbol{\tau}$ with respect to the QPS body frame as follows:
\begin{equation}
    \boldsymbol{\tau}_B=\mathbf{S}\left(\phi,\theta,\psi\right)\boldsymbol{\tau},
\end{equation}
where $\boldsymbol{\tau}_B=\begin{bmatrix}\tau_{\phi}&\tau_{\theta}&\tau_{\psi}\end{bmatrix}^T$. Then, based on Fig. \ref{quad_pic},  squares of rotor angular speeds, denoted by $s_1^2$,  $s_2^2$, $s_3^2$, and $s_4^2$,  can be related to $p$, $\tau_\phi$, $\tau_\theta$, and $\tau_\psi$ with \cite{}
\begin{equation}
    \begin{bmatrix}s_{1}^2\\s_{2}^2\\s_{3}^2\\s_{4}^2\end{bmatrix} = \begin{bmatrix}b&b&b&b\\0&-bl&0&bl\\-bl&0&bl&0\\-k&k&-k&k\end{bmatrix}^{-1}\begin{bmatrix}p\\\tau_\phi\\\tau_\theta\\\tau_\psi\end{bmatrix}
\end{equation}
where $b>0$ and $k>0$ are aerodynamics coefficients and $l>0$ is the length of the quadcopter arm.

\section{Control}
\label{sec:control}
By extending the translational and rotational dynamics of the quadcopter, the motion of the quadcopter can be modeled by
\begin{equation}
\label{generalnonlineardynamics}
\begin{cases}
    \dot{\mathbf{x}}=\mathbf{f}\left(\mathbf{x}\right)+\mathbf{G}\mathbf{u}\\
    \mathbf{y}=\begin{bmatrix}
    x&y&z&\psi
    \end{bmatrix}^T
    \end{cases}
    ,
\end{equation}
where $\mathbf{y}$ is the output vector, and 
\begin{subequations}
\begin{equation}
    \mathbf{x}=\begin{bmatrix}
x&y&z&\dot{x}&\dot{y}&\dot{z}&\phi&\theta&\psi&\dot{\phi}&\dot{\theta}&\dot{\psi}&p&\dot{p}
\end{bmatrix}
^T,
\end{equation}
\begin{equation}
    \mathbf{u}=\begin{bmatrix}
u_{1}&u_{2}&u_{3}&u_{4}
\end{bmatrix}
^T=\begin{bmatrix}
\ddot{p}&\ddot{\phi}&\ddot{\theta}&\ddot{\psi}
\end{bmatrix}^T,
\end{equation}
\end{subequations}
are the state vector and the control input of the quadcopter respectively. In \eqref{generalnonlineardynamics},
\begin{subequations}
\begin{equation}
\mathbf{f}\left(\mathbf{x}\right)=\begin{bmatrix}\dot{x}&
    \dot{y}&
    \dot{z}&
    \left({\frac{p}{m}\hat{\mathbf{b}}_3-{g}\hat{\mathbf{e}}_3}\right)^T&
    \dot{\phi}&
    \dot{\theta}&
    \dot{\psi}&
    0&
    0&
    0&
    \dot{p}&
    0
    \end{bmatrix}^T
,
\end{equation}
\begin{equation}
    \mathbf{G}=
\begin{bmatrix}
\mathbf{0}_{9\times 1}&\mathbf{0}_{9\times 3}\\
\mathbf{0}_{3\times 1}&\mathbf{I}_3\\
0&\mathbf{0}_{1\times 3}\\
1&\mathbf{0}_{1\times 3}\\
\end{bmatrix}
,
\end{equation}
\end{subequations}
are smooth functions obtained in Ref. \cite{rastgoftar2021safe}, where
$\mathbf{I}_3\in \mathbb{R}^{3\times 3}$ is the identity matrix, $\mathbf{0}_{3\times 1}\in \mathbb{R}^{3\times 1}$, $\mathbf{0}_{3\times 3}\in \mathbb{R}^{3\times {3}}${, and $\mathbf{0}_{3\times 9}\in \mathbb{R}^{3\times {9}}$} are the zero-entry matrices.
We use the input-state feedback linearization approach, presented in \cite{rastgoftar2021safe}, for low-level trajectory tracking. To this end, we use state transformation $\mathbf{z}=\mathbf{z}\left(\mathbf{x}\right)$
\begin{equation}
\mathbf{z}(\mathbf{x})=\begin{bmatrix}x&y&z&\dot{x}&\dot{y}&\dot{z}&\ddot{x}&\ddot{y}&\ddot{z}&\dddot{x}&\dddot{y}&\dddot{z}&\psi&\dot{\psi}\end{bmatrix}^T.
\end{equation}
\begin{proposition}
There is a one-to-one transformation between $\mathbf{x}$ and $\mathbf{z}$.
\end{proposition}
\begin{proof}
Note that $x$, $y$, $z$, $\dot{x}$, $\dot{y}$, $\dot{z}$, $\psi$, and $\dot{\psi}$ are the components of vectors $\mathbf{x}$ and $\mathbf{z}$, and the remaining components of $\mathbf{z}$, $\ddot{x}$, $\ddot{y}$, $\ddot{z}$, $\dddot{x}$, $\dddot{y}$, and $\dddot{z}$, can be obtained based on components of state vector $\mathbf{x}$ as follows:
\begin{equation}
    \begin{bmatrix}
        \ddot{x}&
        \ddot{y}&
        \ddot{z}&
        \dddot{x}&
        \dddot{y}&
        \dddot{z}
    \end{bmatrix}^T
    =\mathbf{h}\left(p,\phi,\theta,\psi,\dot{p},\dot{\phi},\dot{\theta},\dot{\psi}\right)=
    \frac{1}{m}
    \begin{bmatrix}
    -{mg}\hat{\mathbf{e}}_3+{p}\hat{\mathbf{b}}_3\\
    {\dot{p}}\hat{\mathbf{b}}_3+p\left({\boldsymbol{\omega}}\times \hat{\mathbf{b}}_3\right) 
    \end{bmatrix}
    .
\end{equation}
On the other hand, $\psi$,   $p$, $\phi$, and $\theta$ can be obtained based on $\ddot{x}$, $\ddot{y}$, and $\ddot{z}$. Also, by taking time derivative from acceleration vector $\ddot{\mathbf{r}}$, we can write
\[
    \dddot{\mathbf{r}}=\frac{1}{m}\dot{p}\hat{\mathbf{b}}_3+\frac{p}{m}\left({\boldsymbol{\omega}}\times \hat{\mathbf{b}}_3\right)  =\begin{bmatrix}\frac{1}{m}\hat{\mathbf{b}}_3&
     -\frac{p}{m}\hat{\mathbf{b}}_2&\frac{p}{m}\hat{\mathbf{d}}_2\times \hat{\mathbf{b}}_3
    \end{bmatrix}
    \begin{bmatrix}\dot{p}&\dot{\phi}&\dot{\theta}\end{bmatrix}^T
   +\frac{p}{m}\dot{\psi}\hat{\mathbf{c}}_3\times \hat{\mathbf{b}}_3
\]
Therefore,  $\dot{p}$, $\dot{\phi}$, and $\dot{\theta}$ are obtained by
\begin{equation}
    \begin{bmatrix}\dot{p}&\dot{\phi}&\dot{\theta}\end{bmatrix}^T=\begin{bmatrix}\frac{1}{m}\hat{\mathbf{b}}_3&
     -\frac{p}{m}\hat{\mathbf{b}}_2&\frac{p}{m}\hat{\mathbf{d}}_2\times \hat{\mathbf{b}}_3
    \end{bmatrix}^{-1}
   \left(\begin{bmatrix}
     \dddot{x}&
        \dddot{y}&
        \dddot{z}
    \end{bmatrix}^T-\frac{p}{m}\dot{\psi}\hat{\mathbf{c}}_3\times \hat{\mathbf{b}}_3\right)
\end{equation}

\end{proof}
Note that $\mathbf{z}$ is updated by
\begin{equation}
    \dot{\mathbf{z}}=\mathbf{A}\mathbf{z}+\mathbf{B}\mathbf{v}
\end{equation}
with
\begin{subequations}
\begin{equation}
    \mathbf{A}=\begin{bmatrix}
    \mathbf{0}_{9\times 3}&\mathbf{I}_9&\mathbf{0}_{9\times 1}&\mathbf{0}_{9\times 1}\\
    \mathbf{0}_{3\times 3}&\mathbf{0}_{3\times 9}&\mathbf{0}_{3\times 1}&\mathbf{0}_{3\times 1}\\
    \mathbf{0}_{1\times 3}&\mathbf{0}_{1\times 9}&0&1\\
    \mathbf{0}_{1\times 3}&\mathbf{0}_{1\times 9}&0&0\\
    \end{bmatrix}
    ,
   \end{equation}
    \begin{equation}
    \mathbf{B}=\begin{bmatrix}
    \mathbf{0}_{9\times 3}&\mathbf{0}_{9\times 1}\\
    \mathbf{I}_{3}&\mathbf{0}_{3\times 1}\\
    \mathbf{0}_{1\times 3}&0\\
    \mathbf{0}_{1\times 3}&1\\
    \end{bmatrix}
    ,
\end{equation}
\end{subequations}
where $\mathbf{v}=\begin{bmatrix}\ddddot{x}&\ddddot{y}&\ddddot{z}&\ddot{\psi}\end{bmatrix}^T$ is related to $\mathbf{u}$ by
\begin{equation}
    \mathbf{v}=\mathbf{M}\mathbf{u}+\mathbf{N},
\end{equation}
where
\begin{subequations}
\begin{equation}
    \mathbf{M}=\left[\begin{array}{ccc|c}
    \frac{1}{m}\hat{\mathbf{b}}_3&-\frac{p}{m}\hat{\mathbf{b}}_{2}&\frac{p}{m}\hat{\mathbf{d}}_2\times \hat{\mathbf{b}}_3& \frac{p}{m}\hat{\mathbf{c}}_{3}\times\hat{\mathbf{b}}_3 \\
    \hline
    0&0&0&1
    \end{array}\right]
    ,
\end{equation}
\begin{equation}
    \mathbf{N}=\dot{\theta}\dot{\psi}\left(\hat{\mathbf{c}}_3\times \hat{\mathbf{d}}_2 \right)+\dot{\phi}\left(\dot{\psi}\hat{\mathbf{c}}_3+\dot{\theta}\hat{\mathbf{d}}_2\right)\times \hat{\mathbf{b}}_1
    .
\end{equation}
\end{subequations}

\textbf{Trajectory Control Design:} We choose 
\begin{equation}
\mathbf{v}=\mathbf{K}\left(\mathbf{z}_d-\mathbf{z}\right)
\end{equation}
where $\mathbf{K}$ is the control gain matrix and  $\mathbf{z}_d(t)=\begin{bmatrix}
 \mathbf{p}^T(t)&\dot{\mathbf{p}}^T(t)&\ddot{\mathbf{p}}^T(t)&\dddot{\mathbf{p}}^T(t)&\psi_d(t)&\dot{\psi}_d(t)
\end{bmatrix}^T$ is the desired state at time $t$ (see Figure \ref{block}). Without loss of generality, we choose $\psi_d(t)=0$ and $\dot{\psi}_d(t)=0$ at any time $t$. We propose a spatiotemporal approach in Section \ref{sec:planning} to obtain desired trajectory $\mathbf{p}\left(t\right)$ that is differentiable and bounded at any time $t$. 
The state vector $\mathbf{z}$ is updated by
\begin{equation}\label{StableExternal}
    \dot{\mathbf{z}}=\left(\mathbf{A}-\mathbf{B}\mathbf{K}\right)\mathbf{z}+\mathbf{K}\mathbf{z}_d
\end{equation}
We choose control gain matrix $\mathbf{K}$ such that matrix $\mathbf{A}-\mathbf{B}\mathbf{K}$ is Hurwitz. Then, the dynamics \eqref{StableExternal} is Bounded-Input-Bounded-Output (BIBO) stable \cite{rastgoftar2021scalable}. 
\begin{figure}[ht]
\centering
\includegraphics[width=6in]{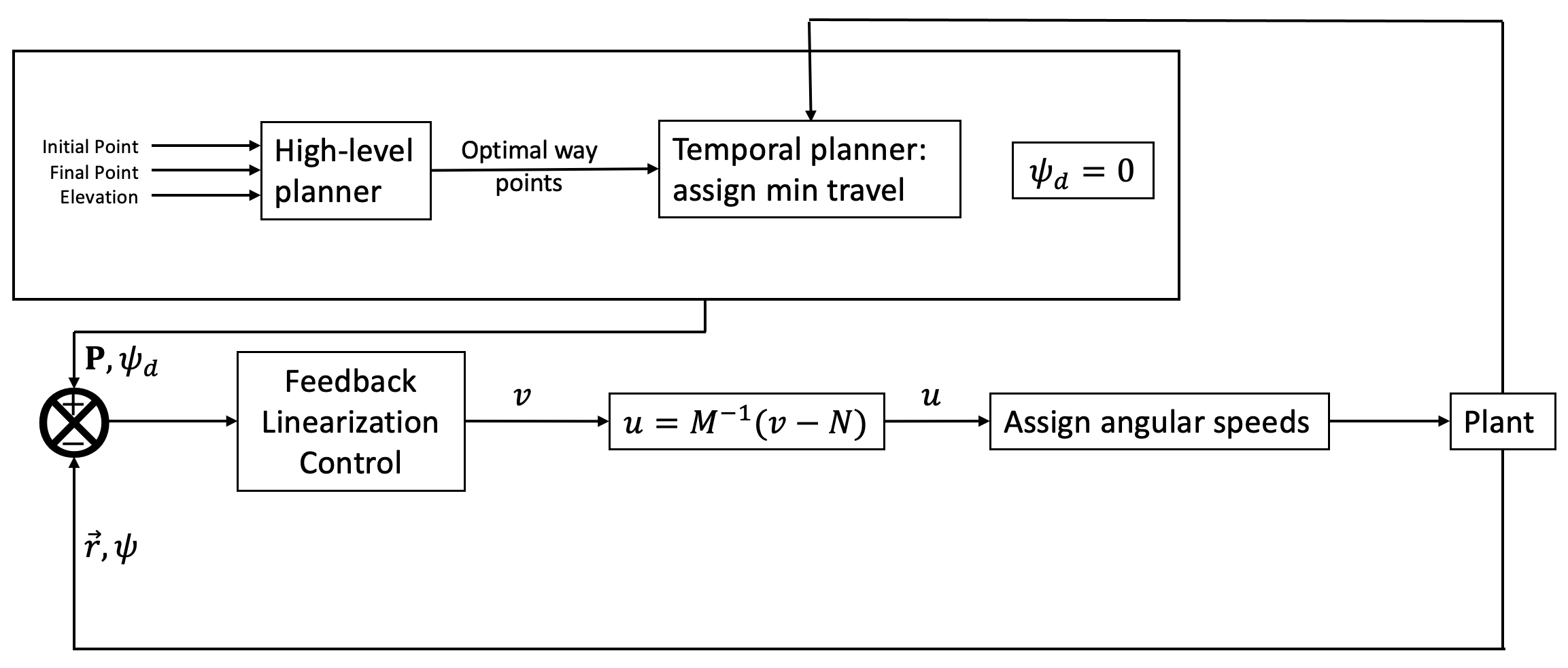}
\caption{Proposed operation for aerial payload transportation}
\label{block}
\end{figure}

\section{Planning}
\label{sec:planning}

We first implement the A* (A star) path-finding algorithm in Section \ref{Path planning} to find a piece-wise straight path which ensures that the system moves from the initial position to the final position while respecting safety condition \eqref{obstacleavoidance}. Then, we implement a bi-section algorithm in Section \ref{Time planning} to find the minimum time required for the QPS to execute each part of the path while respecting safety conditions \eqref{boundedrotor} and \eqref{boundedcontrol}.

\subsection{Spatial planning}
\label{Path planning}
To find a path from the initial position to the final position, that respects the obstacle safety condition, this paper proposes using the weighted A* algorithm  with an exploration grid that is dynamically generated.
To implement the A* search, we discretize the environment with resolution $\Delta$ through the following pairs of functions:

\begin{subequations}
\begin{equation}
\label{F}
    \mathbf{F}: (x,y,z) \rightarrow \left(\floor{\frac{x}{\Delta}+\frac{1}{2}}, \floor{\frac{y}{\Delta}+\frac{1}{2}}, \floor{\frac{z}{\Delta}+\frac{1}{2}}\right),
    \quad \forall (x,y,z) \in \mathbf{Environment},
\end{equation}
\begin{equation}
\label{F-1}
    \mathbf{F^{-1}}: (i,j,k) \rightarrow \left(\Delta \times i,\Delta \times j, \Delta \times k\right).
    \quad \forall (i,j,k) \in F\left(\mathbf{Environment}\right).
\end{equation}
\end{subequations}
Equation \eqref{F} converts continuous-valued point $\left(x,y,z\right)$ to discrete valued point $\left(i,j,k\right)$. On the other hand, Eq. \eqref{F-1} converts dicrete-valued $(i,j,k)$ to an associated discrete-valued point in the environment. 
We perform the following steps to spatially plan a safe path from the initial position to the target destination:
\begin{enumerate}
\item Convert the initial position $\mathbf{r}_i=\begin{bmatrix}x_i&y_i&z_i\end{bmatrix}^T$, final position $\mathbf{r}_f=\begin{bmatrix}x_f&y_f&z_f\end{bmatrix}^T$, and expanded elevation map $M_E$ to initial index, final index, and discrete expanded elevation map $M_E'$ using Function \eqref{F}.
\begin{subequations}
\begin{equation}
(i_1, j_1, k_1) = \mathbf{F}(x_i,y_i,z_i), \quad (i_N, j_N, k_N) = \mathbf{F}(x_f,y_f,z_f),
\end{equation}
\begin{equation}
\label{M_E'}
M_E': (i,j) \rightarrow \max\left(\left\{k\ \vert \ (i,j,k)\in \mathbf{F(Expanded\ obstacle\ space)}\right\}\right).
\end{equation}
\end{subequations}
Note that \eqref{M_E'} defines the discrete expanded elevation map as taking the maximum value in the area covered by indices i and j. And that $N \in \mathbb{N}$ is a finite free variable determined by solving the A* search.
\item Run Algorithm \ref{Astar} with resolution $\Delta$ and weight $w$. The output is a sequence of $N$ indices $\left(i_n,j_n,k_n\right)$ with $n \in \left\{1,\cdots,N\right\}$ that, when connected with segments, form a collision-free path from the initial index to the final index within the grid of resolution $\Delta$.
\item Convert back this sequence of indices to obtain a sequence of $N$ points $\bar{\mathbf{p}}_n$ using Function \eqref{F-1}. These points form a piece-wise straight path between the initial position and the final position which respects the obstacle safety condition as it is contained within the $\mathbf{Restricted\ free\ space}$. This desired trajectory $\mathbf{p}$ is parameterized with parameter $u$ such that:
\begin{subequations}
\begin{equation}
\bar{\mathbf{p}}_n = \mathbf{F^{-1}}(i_n,j_n,k_n) \quad\forall n \in \left\{1,\cdots,N\right\},
\end{equation}
\begin{equation}
\mathbf{p}(u) = 
\begin{cases}
     \bar{\mathbf{p}}_n + (u-n)\times\left(\bar{\mathbf{p}}_{n+1}-\bar{\mathbf{p}}_n\right)\quad n\leq u < n+1 \quad\forall n \in \left\{1,\cdots,(N-1)\right\}.
\end{cases}
\end{equation}
\end{subequations}
\end{enumerate}

\begin{algorithm}
    \caption{Weighted A* search algorithm with dynamically generated grid}\label{Astar}
    \begin{algorithmic}
    \Require {Initial\_index, Final\_index, $M_E'$, $\Delta$, and $w$}
    \Ensure Path is a list of connected indices from initial index to final index through $\mathbf{F(Restricted\ free\ space)}$
        \State index: [integer, integer, integer]\Comment{identifies a cell with indices i, j, and k}
        \State node: [index, real number] \Comment{is a cell with the index of the previous node}
        \State OpenSet $\gets$ HashTable<index,node> \hfill and the distance from the initial node
        \State ClosedSet $\gets$ HashTable<index,node>
        \State OpenSet[Initial\_index] $\gets$ [None, 0]
        \While {OpenSet.keys $\neq \emptyset$ $\wedge$ Final\_index $\notin$ ClosedSet.keys}
            \State PickedIndex $\gets$ min(OpenSet.keys, Index $\rightarrow$ OpenSet[Index][1] + $w \times$ distance(Index,Final\_index))
            
            \Comment{Pick the open node with the lowest total cost}
            \State Neighbors $\gets$ neighbors(PickedIndex)
            
            \Comment{Get neighbors of the picked node}
            \State Neighbors $\gets$ [Neighbor $\in$ Neighbors $\vert$ Neighbor $\notin$ ClosedSet.keys]
            
            \Comment{Filter out closed neighbors}
            \State Neighbors $\gets$ [Neighbor $\in$ Neighbors $\vert$ Neighbor[2] > $M_E'$(Neighbor)]
            
            \Comment{Filter out neighbors under the map}
            \For{Neighbor in Neighbors}
                \If{Neighbor $\notin$ OpenSet.keys}
                    \State OpenSet[Neighbor] $\gets$ [None, $\infty$]
            
                    \Comment{Add new open neighbors}
                \EndIf
                \If{OpenSet[Neighbor][1] > OpenSet[PickedIndex][1] + $w \times$ distance(Neighbor,PickedIndex)}
                    \State OpenSet[Neighbor] $\gets$ [PickedIndex, OpenSet[PickedIndex][1] + $w \times$ distance(Neighbor,PickedIndex)]
            
                    \Comment{Update already open neighbors if needed}
                \EndIf
            \EndFor
            \State ClosedSet[PickedIndex] $\gets$ OpenSet[PickedIndex]
            \State remove(OpenSet.keys, PickedIndex)
            
            \Comment{Close the explored node}
        \EndWhile
        \State Path $\gets$ List<index>
        \State Index $\gets$ Final\_index
        \State i $\gets$ 0
        \While {Index is not None}
            \State Path[i] $\gets$ Index
            \State Index $\gets$ ClosedSet[Index][0]
            \State i $\gets$ i+1
        \EndWhile
        \State Path $\gets$ reverse(Path)\Comment{Reverse the sequence such that it starts from the initial index}
        
        \State $Simpler\_path$ $\gets$ List<index>
        \State $Simpler\_path$[0] $\gets$ Path[0]
        \State j $\gets$ 0
        \For{index in Path[1:]}
            \If{Not connectable($Simpler\_path$[j], index)}
                \State j $\gets$ j+1
                \State $Simpler\_path$[j] $\gets$ index
            \EndIf
        \EndFor
        \State $\mathbf{return}$ $Simpler\_path$
        
        \Comment{Simplify the path by ignoring intermediate points if the previous and next point are directly connectable}

    \end{algorithmic}
\end{algorithm}

\subsection{Temporal planning}
\label{Time planning}
\subsubsection{Full stop condition}
To ensure that the QPS system is able to follow the desired trajectory $\mathbf{p}(t)$ at any time $t$, we impose full stop conditions at $\bar{\mathbf{p}}_1$ through $\bar{\mathbf{p}}_N$. For the controller presented in Section \ref{sec:control}, a full stop is defined as zero velocity, zero acceleration, and zero jerk because it needs these values to be continuous and defined at all times.
As such, we will be using the $\sigma_3$ activation function that satisfies the full stop requirement at initial time $0$ and final time $1$:
\begin{equation}
\sigma_3(t) = -20\times t^7 + 70\times t^6 - 84\times t^5 + 35\times t^4, \qquad \forall t \in \left[0,1\right].
\end{equation}
Such that $\sigma_3(0) = 0, \sigma_3(1) = 1, \dot{\sigma_3}(0) = 0, \dot{\sigma_3}(1) = 0, \ddot{\sigma_3}(0) = 0, \ddot{\sigma_3}(1) = 0, \dddot{\sigma_3}(0) = 0, \dddot{\sigma_3}(1) = 0$.

\bigskip

Now, we  define $t_n$ as the time of arrival at the $n$-th point $\bar{\mathbf{p}}_n$ for all n in $\left\{1,\cdots,N\right\}$. By using the activation function $\sigma_3$, the time parameterization of the trajectory becomes:

\begin{equation}
\mathbf{p}(t) = 
\begin{cases}
     \bar{\mathbf{p}}_n + \sigma_3\left(\frac{t-t_n}{t_{n+1}-t_n}\right)\times\left(\bar{\mathbf{p}}_{n+1}-\bar{\mathbf{p}}_n\right)\quad t_n\leq t < t_{n+1} \quad\forall n \in \left\{1,\cdots,(N-1)\right\}
\end{cases}
\end{equation}

\subsubsection{Travel Time minimization}
We use the bi-section method to find the minimum time required for each part of the trajectory  $\mathbf{p}$ such that it is achievable with valid rotor speeds $s_1$ through $s_4$ and bounded error $\|\mathbf{r}(t)-\mathbf{p}(t)\|$, i.e. safety conditions \eqref{boundedrotor} and \eqref{boundedcontrol} are satisfied.

For each segment and time guessed for it, we can run a simulation with the controller presented in Section \ref{sec:control} that checks if the valid rotor speed and bounded error safety conditions were violated. This process can be modeled with a simulation function $test$ defined as follows:
\begin{equation}
test(t) =
\begin{cases}
Valid&\text{if the safety conditions were not violated}\\
Invalid&\text{if the safety conditions were violated}\\
\end{cases}, \quad \forall t.
\end{equation}

The bi-section algorithm \ref{timealgorithm} works in two phases. In the first phase, it tries to find a valid maximum guess $t_{max}$ by iteratively doubling and testing an initial guess while setting the minimum guess $t_{min}$ to the previous invalid guess. Once a valid maximum guess has been found, it moves to phase two where it iteratively refines its guess range $\left[t_{min}, t_{max}\right]$. It does this by testing the midpoint of the range $t_{mid}$. If the test result is valid, the midpoint becomes the new maximum guess $t_{max}$. If the test result is invalid, the midpoint becomes the new minimum guess $t_{min}$.

\bigskip

The second phase of the bi-section algorithm keeps running till the guess range satisfies some condition. For this paper, the condition is \eqref{time condition}. Once the condition is achieved, the maximum of the guess range is picked as it is the only time in the guess range that was verified as valid.
\begin{equation}
\label{time condition}
\frac{t_{max}-t_{min}}{t_{mid}} \leq \delta_t
\end{equation}
where $\delta_t$ is the given time percentage error.

\bigskip

\begin{algorithm}
    \caption{Bi-section algorithm}\label{timealgorithm}
    \begin{algorithmic}
    \Require{initial guess $t_{max}$, test function $test$, and time percentage error $\delta_t$}
    \Ensure{$test(t_{max})$ is Valid and $t_{max}$ is within a range [$t_{min}$,$t_{max}$] that respects \eqref{time condition} while $test(t_{min})$ is Invalid}

        \State $t_{min}$ $\gets$ 0
    
        \While {$test(t_{max})$ is Invalid}
            \State $t_{min}$ $\gets$ $t_{max}$
            \State $t_{max}$ $\gets$ $2\times t_{max}$
        \EndWhile
        
        \State $t_{mid}$ $\gets$ $\frac{t_{max}+t_{min}}{2}$
        
        \While {$\frac{t_{max}-t_{min}}{t_{mid}} > \delta_t$}
            \State $t_{mid}$ $\gets$ $\frac{t_{max}+t_{min}}{2}$
            \If{$t_{mid}$ is valid}
                \State $t_{max}$ $\gets$ $t_{mid}$
            \Else
                \State $t_{min}$ $\gets$ $t_{mid}$
            \EndIf
        \EndWhile
    \end{algorithmic}
\end{algorithm}

\section{Simulation}
\label{sec:simulation}

\begin{table}
\centering
\caption{Parameters of QPS used for simulation. The quadcopter parameters are selected from \cite{romano2019experimental}.}
    \begin{tabular}{|c|c|c|}
    \hline
         Parameter&Value  &Unit\\
         \hline
         $m$&$0.5$  &$kg$\\
         $g$&  $9.81$ &$m/s^2$\\
         $l$&  $0.25$  &$m$\\
         $J_{x}$& $0.0196$ &$kg~m^2$\\
         $J_{y}$&$0.0196 $  &$kg~m^2$\\
         $J_{z}$&$0.0264$ &$kg~m^2$\\
         $b$& $3\times 10^{-5}$ &$N~s^2/rad^2$\\
         $k$&$1.1\times 10^{-6}$&$N~s^2/rad^2$\\
         \hline
    \end{tabular}
    \label{quadparameters}
\end{table}

\begin{table}
\centering
\caption{Parameters of payload used for simulation.}
    \begin{tabular}{|c|c|c|}
    \hline
         Parameter&Value  &Unit\\
         \hline
         $m$&$0.3$  &$kg$\\
         $J_{x}$&$0.005$ &$kg~m^2$\\
         $J_{y}$&$0.005$ &$kg~m^2$\\
         $J_{z}$&$0.005$ &$kg~m^2$\\
         \hline
    \end{tabular}
    \label{payloadparameters}
\end{table}

We considered a QPS modeled by kinematics and dynamics covered in Section \ref{sec:Quadcopter-payload system model} performing a payload transport mission through the University of Arizona. The quadcopter's parameters are listed in Table \ref{quadparameters} and the payload's parameters are listed in Table \ref{payloadparameters}.

\bigskip

Assuming that the payload is rigidly attached under the drone such that their centers of mass are at a distance of $d=0.2\ m$, and that both have diagonal inertia matrices, we can compute the following parameters for the QPS using the parallel axis theorem:
\begin{subequations}
\begin{equation}
m = 0.5\ kg +0.3\ kg = 0.8\ kg,
\end{equation}
\begin{equation}
d' = 0.2\ m\times\frac{0.3\ kg}{0.3\ kg+0.5\ kg} = 0.075\ m,
\end{equation}
\begin{equation}
J_x = 0.0196\ kg\ m^2+0.5\ kg\times\left(0.075\ m\right)^2 + 0.005\ kg\ m^2+0.3\ kg\times\left(0.2\ m-0.075\ m\right)^2=0.035225\ kg\ m^2,
\end{equation}
\begin{equation}
J_y = 0.0196\ kg\ m^2+0.5\ kg\times\left(0.075\ m\right)^2 + 0.005\ kg\ m^2+0.3\ kg\times\left(0.2\ m-0.075\ m\right)^2=0.035225\ kg\ m^2,
\end{equation}
\begin{equation}
J_z = 0.0264\ kg\ m^2 + 0.005\ kg\ m^2 = 0.0314\ kg\ m^2,
\end{equation}
\end{subequations}
where $d'$ is the distance between the QPS's center of mass and the quadcopter's center of mass.

\bigskip

The mission was defined with starting and ending points in WGS84 coordinates (latitude, longitude, altitude above mean sea level in meters (AMSL)). They were then converted to Universal Transverse Mercator (utm) x, y, z coordinates in meters with center of coordinates (32.2333°, -110.9504°, 0).

\bigskip

Starting point: (32.2318344°, -110.9543101°, 753)

Ending point: (32.2380538°, -110.9486297°, 758)

\bigskip

The elevation map of the University of Arizona was generated using Lidar data from USGS. Points from the Lidar point clouds were first converted to WGS84 coordinates then to utm x, y, z coordinates. Finally, the data was interpolated to create a continuous elevation map which was expanded with $\epsilon=0.65\ m$ and $\delta=0.35\ m$. The discrete version of the expanded elevation map is generated  with $\Delta = 1\ m$ (see Eq. \ref{F}). 
\begin{table}
\centering
\caption{Trajectory waypoints for payload transport mission with $(\Delta = 1\ m,\ w = 1.1)$ and utm center of coordinates (32.2333°, -110.9504°, 0) with arrival time $t_n$ computed through the bi-section algorithm with $\delta_t = 0.05$ and $s_{max}=400\ rad/s$} 
    \begin{tabular}{|c|c|c|c|c|c|}
    \hline
        Latitude&Longitude&AMSL&utm x&utm y&$t_n$ (s)\\
        \hline
        32.2318344°&-110.9543101°&753&-368&-162&0\\
        32.2320113°&-110.9540943°&756&-348&-143&9.6\\
        32.2320203°&-110.9540836°&757&-347&-142&13.0\\
        32.2320293°&-110.9540730°&758&-346&-141&16.3\\
        32.2320383°&-110.9540624°&759&-345&-140&23.1\\
        32.2330393°&-110.9528838°&759&-234&-29&48.7\\
        32.2330483°&-110.9528731°&760&-233&-28&52.0\\
        32.2336525°&-110.9521617°&762&-166&39&69.6\\
        32.2342206°&-110.9514927°&762&-103&102&88.0\\
        32.2380538°&-110.9486297°&758&166&527&150.4\\

         \hline
    \end{tabular}
    \label{waypoints}
\end{table}
\begin{figure}[ht]
\centering
\includegraphics[width=6.5   in]{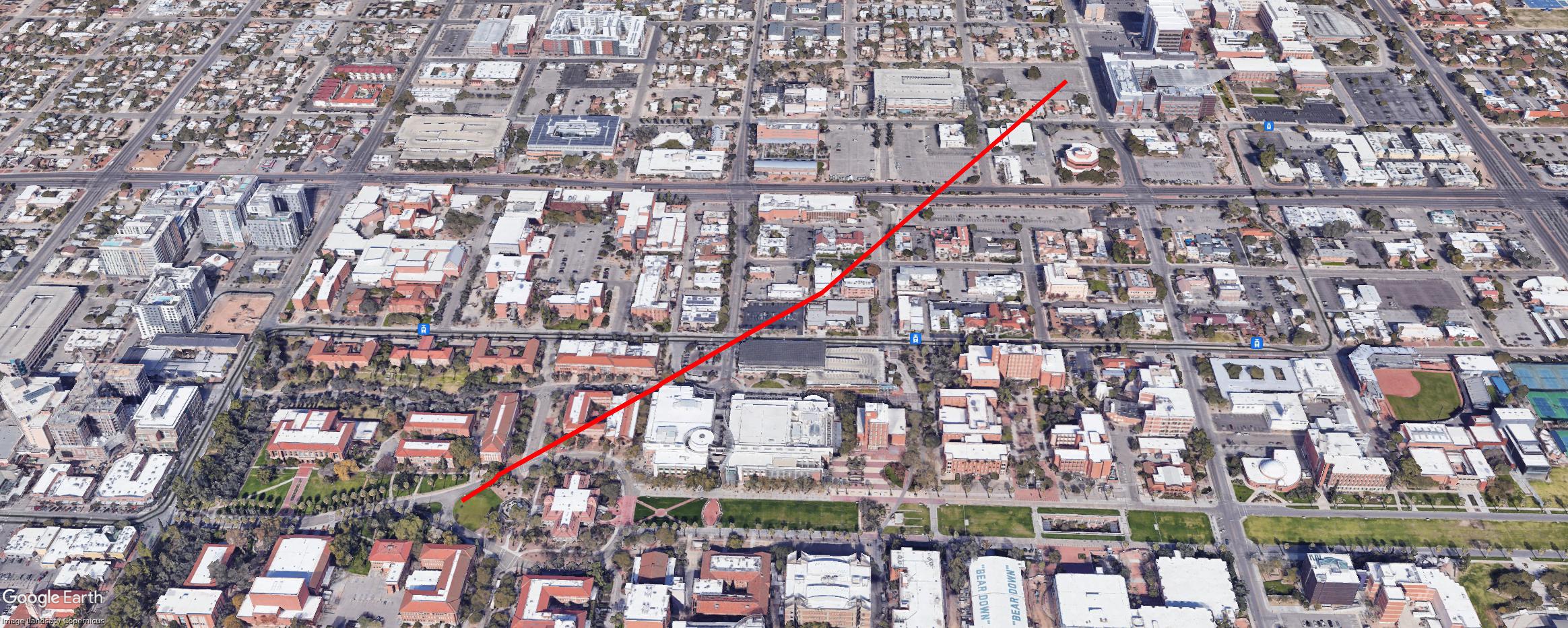}
\caption{Simulation of the aerial payload transportation in the Univeristy of Arizona campus (Optimal path for payload delivery shown by red)}
\label{trajectory}
\end{figure}
\begin{figure}
 \centering
 \includegraphics[width=4.5 in]{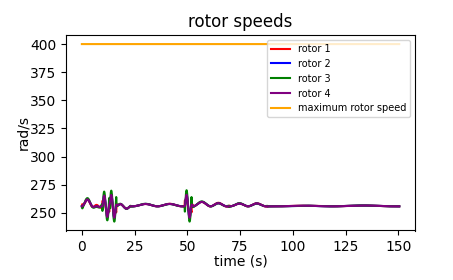}
     \caption{Angular Rotor speeds for the mission computed through the controller.}
\label{AngularSpeeds}
\end{figure}
\begin{figure}
 \centering
 \includegraphics[width=4.5 in]{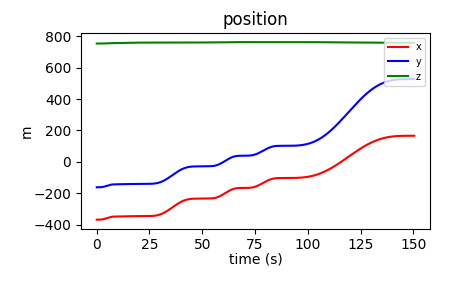}
     \caption{Position of system as a function of time during simulation.}
\label{trajplot}
\end{figure}
We ran our path-finding algorithm described in Section \ref{Path planning} with $(\Delta = 1\ m,\ w = 1.1)$ and obtained the trajectory shown in Table \ref{waypoints}. Fig. \ref{trajectory} shows that the trajectory does indeed avoids all structures and terrain. We then used our bi-section algorithm with $\delta_t = 0.05$ and $s_{max}=400\ rad/s$ to determine the minimum times required for each part of the mission which are shown in Table \ref{waypoints}. Finally, we ran our controller with a simulation which provided the motor speeds presented in Fig. \ref{AngularSpeeds} and positions presented in Fig. \ref{trajplot}.
\clearpage

\section{Conclusion}
\label{sec:conclusion}

We studied the problem of payload transportation by a single quadcopter in urban areas by considering the combined system of quadcopter and payload together as a rigid body. We used USGS  LIDAR data to generate an elevation map  for the University of Arizona. We then applied a hybrid approach, integrating a high-level motion planner with a low-level trajectory control, to safely plan a payload transportation mission assuring collision avoidance and boundedness of rotor angular speeds and trajectory tracking. In particular, the high-level motion planner combines the A* search with polynomial planning to obtain a collision-free desired trajectory minimizing travel distance from the initial position to the target destination.  We showed that the quadcopter can stably track the desired trajectory by applying the input-output feedback linearization control.


\section*{Acknowledgments}
This work has been supported by the National Science Foundation under Award Nos. 2133690 and 1914581. 
\bibliography{sample}

\end{document}